\pdfoutput=1
\documentclass[conference]{IEEEtran}

\usepackage{cite}
\usepackage{amsmath,amssymb,amsfonts,mathtools}
\usepackage{amsthm}
\usepackage{array}
\usepackage{booktabs}
\usepackage{tabularx}
\usepackage{algorithm}
\usepackage{algorithmic}
\usepackage{graphicx}
\usepackage{textcomp}
\usepackage{xcolor}
\usepackage{microtype}
\usepackage{tikz}
\usetikzlibrary{arrows.meta,positioning,fit}
\usepackage[hidelinks]{hyperref}

\microtypesetup{protrusion=true,expansion=true}
\newif\ifanonymous
\anonymousfalse
\newtheorem{theorem}{Theorem}[section]
\newtheorem{lemma}[theorem]{Lemma}
\newtheorem{proposition}[theorem]{Proposition}
\newtheorem{corollary}[theorem]{Corollary}
\theoremstyle{definition}
\newtheorem{definition}[theorem]{Definition}
\newtheorem{remark}[theorem]{Remark}
\newtheorem{assumption}[theorem]{Assumption}

\newcommand{\R}{\mathbb{R}}
\newcommand{\E}{\mathbb{E}}

\newcommand{\norm}[1]{\left\lVert #1 \right\rVert}
\newcommand{\ip}[2]{\left\langle #1,#2 \right\rangle}
\newcommand{\defeq}{\mathrel{\mathop:}=}

\newcommand{\diag}{\mathrm{diag}}

\newcommand{\placeholderbox}[2][1.55in]{%
  \fbox{\parbox[c][#1][c]{0.95\columnwidth}{\centering\textbf{Placeholder figure}\\[2mm]#2}}
}

\def\BibTeX{{\rm B\kern-.05em{\sc i\kern-.025em b}\kern-.08em
    T\kern-.1667em\lower.7ex\hbox{E}\kern-.125emX}}

\begin{document}

\title{Sharpness Aware Surrogate Training for Spiking Neural Networks}

\ifanonymous
\author{%
\IEEEauthorblockN{Anonymous Submission}
\IEEEauthorblockA{Author information omitted for review}
}
\else
\author{%
\IEEEauthorblockN{Maximilian Nicholson}
\IEEEauthorblockA{University of Bath, United Kingdom\\
mn866@bath.ac.uk}
}
\fi

\maketitle

\begin{abstract}
\noindent Surrogate gradients are a standard tool for training spiking neural networks (SNNs), but conventional hard forward or surrogate backward training couples a nonsmooth forward model with a biased gradient estimator. We study sharpness aware Surrogate Training (SAST), which applies sharpness aware Minimization (SAM) to a surrogate forward SNN trained by backpropagation. In this formulation, the optimization target is an ordinary smooth empirical risk, so the training gradient is exact for the auxiliary model being optimized. Under explicit boundedness and contraction assumptions, we derive compact state stability and input Lipschitz bounds, establish smoothness of the surrogate objective, provide a first order SAM approximation bound, and prove a nonconvex convergence guarantee for stochastic SAST with an independent second minibatch. We also isolate a local mechanism proposition, stated separately from the unconditional guarantees, that links per sample parameter gradient control to smaller input gradient norms under local Jacobian conditioning. Empirically, we evaluate clean accuracy, hard spike transfer, corruption robustness, and training overhead on N-MNIST and DVS Gesture. The clearest practical effect is transfer gap reduction: on N-MNIST, hard spike accuracy rises from 65.7\% to 94.7\% (best at $\rho=0.30$) while surrogate forward accuracy remains high; on DVS Gesture, hard spike accuracy improves from 31.8\% to 63.3\% (best at $\rho=0.40$). We additionally specify the compute matched, calibration, and theory alignment controls required for a final practical assessment.
\end{abstract}

\section{Introduction}
\noindent Spiking neural networks (SNNs) are attractive for neuromorphic hardware because they naturally support sparse, event based computation \cite{maass1997networks,indiveri2015memory,roy2019towards,merolla2014truenorth,davies2018loihi}. Event based sensors such as dynamic vision sensors (DVS) produce temporally sparse streams that align well with SNN processing and architectur \cite{lichtsteiner2008dvs,gallego2022eventvision}. The central challenge is supervised training: spikes arise from a discontinuous threshold, so standard backpropagation is not directly applicable. Surrogate-gradient methods address this by replacing the derivative of the spike function with a smooth proxy during optimization \cite{bengio2013stochastic,lee2016training,rueckauer2017conversion,neftci_surrogate,zenke_superspike,bellec2018l2l,sengupta2019goingdeeper,eshraghian2023lessons,shrestha2018slayer,wu2018stbp}. A persistent practical issue is the surrogate to hard transfer gap. During training, the surrogate forward network produces graded activations; at deployment, one may replace the surrogate by a hard threshold to obtain a true spiking model. When many membrane potentials lie near threshold, this replacement can change firing decisions sharply. Those local changes then accumulate over time and across layers, so a model that looks accurate under surrogate forward evaluation can perform much worse under hard spike inference. In this paper, we treat that gap as a first class metric rather than an afterthought.

Even with surrogate gradients, multi-step SNN training can remain fragile. Temporal recurrence amplifies sensitivity to perturbations and optimization noise, especially when the unrolled dynamics behave like a stiff recurrent system \cite{werbos1990bptt,pascanu2013difficulty,zenke2021robustness}. A useful viewpoint is to treat the unrolled SNN as a discrete-time nonlinear state-space model: each layer combines a stable linear recurrence with a static nonlinearity and reset feedback. This view makes explicit how gain accumulates across time and layers and motivates asking whether flatter solutions in parameter space also improve robustness in input space. Sharpness aware Minimization (SAM) optimizes a neighborhood worst-case loss and has been effective for improving generalization in conventional deep networks \cite{keskar2017largebatch,dinh2017sharpminima,foret_sam,kwon2021asam,du2021esam}. Existing analyses clarify how SAM relates to sharpness and to gradient-norm regularization \cite{wen2023samsharpness}. The question here is whether the same idea helps surrogate trained SNNs in a way that is both algorithmically simple and analytically defensible.

Our approach is to apply SAM not to a hard spike forward pass with surrogate derivatives in the backward pass, but to a surrogate forward SNN whose forward dynamics already use a smooth spike approximation. This removes the hard forward or surrogate backward mismatch and lets backpropagation through time compute the exact gradient of a smooth objective. As a result, the analysis is tied to the objective that is actually optimized during training, which keeps both assumptions and guarantees technically clean. It also makes the theory appropriately narrow: we analyze the smooth auxiliary model that SAST actually optimizes, rather than claiming guarantees for generic straight-through SNN training.

\subsection{Related work and positioning}
Surrogate-gradient training is the dominant supervised approach for modern SNNs, including multilayer recurrent and convolutional settings trained with backpropagation through time \cite{neftci_surrogate,zenke_superspike,bellec2018l2l,shrestha2018slayer,wu2018stbp,eshraghian2023lessons}. Our paper stays within that family but changes the optimization target: the forward model itself is smoothed, so the gradient is exact for the auxiliary model that is trained. SAM and its variants have mainly been analyzed in conventional deep networks \cite{foret_sam,kwon2021asam,du2021esam,wen2023samsharpness}. We borrow the SAM objective and one-step ascent approximation, but our analysis must account for temporal recurrence, reset feedback, and the surrogate-state dynamics specific to SNNs. We therefore pair the optimization method with explicit state stability and input-Lipschitz bounds for the surrogate forward model. A separate line of work controls sensitivity more directly through gradient penalties or related regularization \cite{drucker1992doublebackprop,ross2018inputgrad,sokolic2017robust,novak2018sensitivity}. We use such methods as diagnostics or baselines rather than as primary objects of analysis. In particular, we treat compute-matched training, ASAM, input gradient regularization, and post-training threshold calibration as important controls when assessing whether SAST offers value beyond cheaper gap-reduction heuristics.

\textbf{Contributions.}
\begin{itemize}
    \item We introduce \emph{sharpness aware Surrogate Training} (SAST), a simple integration of SAM into surrogate forward SNN training \cite{foret_sam,kwon2021asam}.
    \item We formalize a multi-layer leaky integrate-and-fire (LIF) SNN and its smooth surrogate forward counterpart. Training uses the surrogate forward model, while inference may optionally use either the surrogate model or the corresponding hard threshold model.
    \item Under explicit boundedness and contraction assumptions, we derive compact state stability and input Lipschitz bounds, prove smoothness of the surrogate empirical objective, and obtain a first order SAM approximation bound together with a nonconvex convergence guarantee for stochastic SAST.
    \item We isolate a separate local mechanism proposition: under local Jacobian conditioning, smaller per sample parameter gradients upper bound smaller input gradient norms. We present this as a testable explanation for sensitivity trends, not as an unconditional consequence of SAM.
    \item We define a rigorous evaluation protocol centered on surrogate forward accuracy, hard spike accuracy, transfer gap, corruption robustness, hard spike evaluation fairness, compute-matched controls, and theory-to-practice diagnostics such as measured contraction factors and local smoothness secants.
\end{itemize}

\section{Model and Training Objective}
\subsection{Notation}
\noindent We consider sequences $x_{1:T}=(x_1,\dots,x_T)$ with $x_t\in\R^{d_0}$, labels $y\in\{1,\dots,C\}$, and parameters $w\in\R^p$. For a linear operator $A$, $\norm{A}_2$ denotes its spectral norm. Define the geometric temporal factor
\begin{equation}
S_T(\alpha) \defeq \sum_{k=0}^{T-1}\alpha^k = \frac{1-\alpha^T}{1-\alpha},\qquad \alpha\in(0,1).
\label{eq:ST}
\end{equation}
For sequence differences we use
\begin{equation}
\norm{x_{1:T}-x'_{1:T}}_{2,2} \defeq \left(\sum_{t=1}^T \norm{x_t-x'_t}_2^2\right)^{1/2}.
\end{equation}

\subsection{hard spike and surrogate forward SNNs}
We define an $L$-layer SNN unrolled for $T$ time steps. Each layer $\ell$ has state dimension $d_\ell$.

\begin{definition}[hard spike LIF SNN with reset-by-subtraction]
\label{def:lif}
Fix leak $\alpha\in(0,1)$. Let $u_0^{(\ell)}=0$ and $s_0^{(\ell)}=0$. Let $z_t^{(0)}\defeq x_t$. For $\ell=1,\dots,L$ and $t=1,\dots,T$,
\begin{align}
u_t^{(\ell)} &= \alpha u_{t-1}^{(\ell)} + A^{(\ell)} z_t^{(\ell-1)} + b^{(\ell)} - \theta^{(\ell)}\odot s_{t-1}^{(\ell)}, \label{eq:lif_u}\\
s_t^{(\ell)} &= H\!\left(u_t^{(\ell)} - \theta^{(\ell)}\right), \label{eq:lif_s}\\
z_t^{(\ell)} &\defeq s_t^{(\ell)}.
\end{align}
The readout averages the last-layer output over time:
\begin{equation}
o \defeq f_w(x_{1:T}) \defeq W_{\text{out}}\left(\frac{1}{T}\sum_{t=1}^{T} z_t^{(L)}\right)+b_{\text{out}} \in\R^C.
\label{eq:readout}
\end{equation}
\end{definition}

\begin{remark}[theory aligned linear blocks]
In experiments, $A^{(\ell)}$ may represent a composition of convolution, normalization, and pooling \cite{ioffe2015batchnorm}. For the smooth analysis below, we treat each block as a bounded linear map,
\begin{equation}
A^{(\ell)} \defeq P^{(\ell)} B^{(\ell)} W^{(\ell)},
\end{equation}
where $W^{(\ell)}$ is convolution, $B^{(\ell)}$ is the affine scaling part of an inference mode normalization layer, and $P^{(\ell)}$ is average pooling or another linear downsampling operator. Any additive normalization offset is absorbed into $b^{(\ell)}$. If the actual implementation uses MaxPool or training mode BatchNorm, the smoothness theory no longer applies verbatim; those cases are reported as practical extensions rather than direct confirmations of the theory.
\end{remark}

\begin{definition}[Admissible surrogate nonlinearity]
\label{def:admissible}
A surrogate $\sigma:\R\to[0,1]$ is \emph{admissible} if $\sigma\in C^2$, is nondecreasing, and satisfies
\begin{equation}
\sup_{x\in\R}|\sigma'(x)| \le B_1,
\qquad
\sup_{x\in\R}|\sigma''(x)| \le B_2
\label{eq:B1B2}
\end{equation}
for finite constants $B_1,B_2$. Common examples include arctan and fast sigmoid surrogates \cite{neftci_surrogate,zenke_superspike,shrestha2018slayer,wu2018stbp,eshraghian2023lessons}.
\end{definition}

\begin{definition}[surrogate forward SNN]
\label{def:surrogate}
Replace \eqref{eq:lif_s} by
\begin{equation}
\tilde s_t^{(\ell)} = \sigma\!\left(u_t^{(\ell)}-\theta^{(\ell)}\right),
\qquad
\tilde z_t^{(\ell)} \defeq \tilde s_t^{(\ell)},
\label{eq:sur_spike}
\end{equation}
while keeping the membrane recursion \eqref{eq:lif_u} with $z_t^{(\ell)}$ replaced by $\tilde z_t^{(\ell)}$ and $s_{t-1}^{(\ell)}$ replaced by $\tilde s_{t-1}^{(\ell)}$. The readout remains \eqref{eq:readout}, producing logits $\tilde f_w(x_{1:T})$.
\end{definition}

\begin{remark}[Scope of the theory]
\label{rem:train_infer}
Throughout the analysis, SAST trains the surrogate forward network of Definition~\ref{def:surrogate}. In that setting, backpropagation through time computes the exact gradient of the smooth empirical objective. The commonly used hard forward and surrogate backward estimator is generally biased and is not covered by the smoothness or convergence results stated here \cite{bengio2013stochastic,neftci_surrogate}.
\end{remark}

\begin{definition}[hard spike evaluation protocol]
\label{def:hard_eval}
Given trained surrogate forward parameters $w$, \emph{hard spike evaluation} keeps the learned weights, biases, thresholds, leak, reset-by-subtraction rule, time discretization, and readout fixed and replaces $\sigma$ by $H$ in every layer. All hidden states are reset at the start of each input sequence. Unless a calibration baseline is explicitly stated, no post-hoc threshold rescaling, clipping, or temperature fitting is applied.
\end{definition}

\begin{remark}[Why surrogate-to-hard transfer can degrade]
Definition~\ref{def:hard_eval} changes only the spike nonlinearity, but that change can still be large in effect. If many membrane potentials lie close to threshold, a smooth surrogate can output intermediate activations while the hard model must decide between $0$ and $1$. The resulting firing-rate mismatch compounds across time and layers, which is why we report both surrogate forward and hard spike accuracy throughout.
\end{remark}

For each layer $\ell$, let $\Theta^{(\ell)}\defeq \diag(\theta^{(\ell)})$. Under the surrogate forward dynamics,
\begin{align}
    u_t^{(\ell)} &= \alpha u_{t-1}^{(\ell)} + A^{(\ell)} \tilde z_t^{(\ell-1)} + b^{(\ell)} - \Theta^{(\ell)}\tilde z_{t-1}^{(\ell)}, \label{eq:ssm_u}\\
    \tilde z_t^{(\ell)} &= \sigma\!\left(u_t^{(\ell)}-\theta^{(\ell)}\right). \label{eq:ssm_z}
\end{align}
The readout is linear in the time average of the final-layer output,
\begin{equation}
\bar z^{(L)}\defeq \frac{1}{T}\sum_{t=1}^T \tilde z_t^{(L)},
\qquad
\tilde f_w(x_{1:T}) = W_{\mathrm{out}}\bar z^{(L)} + b_{\mathrm{out}}.
\end{equation}
This emphasizes the surrogate forward SNN as a stable affine recurrence composed with a smooth static nonlinearity and reset feedback.

\begin{figure*}[t]
\centering
\resizebox{0.86\textwidth}{!}{%
\begin{tikzpicture}[
  font=\footnotesize,
  box/.style={draw, rounded corners, align=center, inner sep=4pt, minimum height=9mm, text width=2.7cm},
  arr/.style={-Latex, thick},
  node distance=0.7cm
]
\node[box] (x) {Input sequence\\$x_{1:T}$};
\node[box, right=of x] (snn) {Unrolled $L$-layer\\surrogate LIF SNN};
\node[box, right=of snn] (avg) {Time average\\$\bar z=\frac{1}{T}\sum_{t=1}^T z_t^{(L)}$};
\node[box, right=of avg] (out) {Readout\\$o=W_{\text{out}}\bar z+b_{\text{out}}$};
\draw[arr] (x) -- (snn);
\draw[arr] (snn) -- (avg);
\draw[arr] (avg) -- (out);
\end{tikzpicture}%
}
\caption{Overview of the surrogate forward SNN used during training. At inference, the final trained model can be evaluated either with the surrogate nonlinearity or by replacing it with a hard threshold according to Definition~\ref{def:hard_eval}.}
\label{fig:network_overview}
\end{figure*}
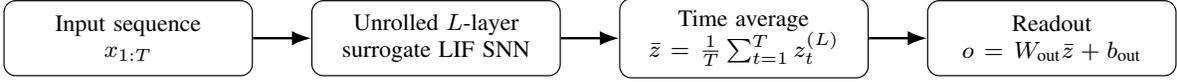

\subsection{Loss and empirical objective}
\noindent Let $\ell(o,y)$ denote multiclass cross-entropy,
\begin{equation}
\ell(o,y) \defeq -o_y + \log\!\left(\sum_{c=1}^C e^{o_c}\right).
\end{equation}
The empirical surrogate objective is
\begin{equation}
\tilde L_S(w) \defeq \frac{1}{n}\sum_{i=1}^n \ell\big(\tilde f_w(x^{(i)}_{1:T}),y^{(i)}\big).
\label{eq:emp_obj}
\end{equation}

\section{Method: sharpness aware Surrogate Training}
SAST applies SAM to the surrogate objective \eqref{eq:emp_obj} \cite{foret_sam,kwon2021asam}. In both SAM passes, the loss is evaluated by unrolling the surrogate forward dynamics and backpropagating through time:
\begin{equation}
\tilde L_{\mathrm{SAM}}(w)\defeq \max_{\norm{\epsilon}_2\le\rho}\tilde L_S(w+\epsilon).
\end{equation}
Algorithm~\ref{alg:sast} uses the standard one-step ascent approximation.

\paragraph{Implementation details used throughout.}
\noindent Unless noted otherwise, the experimental implementation matches the theory as closely as possible.
\begin{itemize}
    \item \textbf{surrogate forward training.} Use a smooth surrogate in the \emph{forward} pass, not only in the backward pass. In theory aligned runs we use the arctan surrogate
    \begin{equation}
        \sigma(x)=\frac{1}{2}+\frac{1}{\pi}\arctan(kx), \qquad k>0,
        \label{eq:arctan_surrogate}
    \end{equation}
    which satisfies Definition~\ref{def:admissible} with
    \begin{equation}
        B_1=\frac{k}{\pi},
        \qquad
        B_2=\frac{3\sqrt{3}}{8\pi}k^2.
        \label{eq:arctan_B1B2}
    \end{equation}
    Because the theory requires the contraction condition $\alpha+M_\theta B_1<1$ below, the surrogate slope $k$ is reported explicitly. theory aligned runs should either satisfy that condition under the observed threshold bound or else be labeled as practical extensions.

    \item \textbf{SAM state handling.} Reset all SNN states at the start of each SAM pass. This matters because the first perturbation pass changes the parameter vector.

    \item \textbf{BPTT policy.} Use full BPTT over the chosen window length $T$ unless a truncated variant is itself part of the experimental study.

    \item \textbf{Second minibatch.} For the convergence theorem, use an independent second minibatch $B'\neq B$ when computing the SAM update direction. If $B'=B$ is used in practice, report it as a separate ablation rather than silently substituting it.

    \item \textbf{No hidden hard-evaluation calibration.} Unless a calibration baseline is explicitly reported, hard spike evaluation follows Definition~\ref{def:hard_eval} exactly, with no threshold rescaling or test-time fitting.

    \item \textbf{Theory-to-practice diagnostics.} For theorem-aligned runs, report the observed threshold bound $\hat M_\theta\defeq \max_{\ell,k}\norm{\theta_k^{(\ell)}}_\infty$, the implied contraction diagnostic $\hat\gamma\defeq \alpha + \hat M_\theta B_1$, and at least one empirical local smoothness diagnostic around trained checkpoints.

    \item \textbf{Evaluation modes.} Report both surrogate forward accuracy $\tilde f_w$ and hard spike accuracy $f_w$ to quantify the surrogate-to-hard transfer gap.
\end{itemize}

\begin{algorithm}[t]
\caption{SAST: SAM for surrogate forward SNN training}
\label{alg:sast}
\begin{algorithmic}[1]
\STATE \textbf{Input:} parameters $w$, optimizer $\mathcal{O}$, radius $\rho$, minibatches $B=(x_{1:T},y)$ and $B'=(x'_{1:T},y')$
\STATE Compute surrogate loss $\tilde L_B(w)$ and gradient $g \leftarrow \nabla_w \tilde L_B(w)$
\STATE Set $\epsilon \leftarrow \rho\, g / (\norm{g}_2 + \delta)$
\STATE Form perturbed parameters $w' \leftarrow w + \epsilon$
\STATE Reset all SNN states and compute $g' \leftarrow \nabla_w \tilde L_{B'}(w')$
\STATE Update $w \leftarrow \mathcal{O}(w,g')$
\end{algorithmic}
\end{algorithm}

\section{Theory in the Main Paper}
\noindent We keep the main text focused on the practical guarantees that support the method: explicit state stability, an input Lipschitz bound, smoothness of the surrogate objective, a first order SAM approximation, and a stochastic convergence guarantee. Conservative closed form constants, the local Jacobian mechanism proposition, and proof sketches appear in Appendix~\ref{app:theory}.

\subsection{Setup assumptions}
\begin{assumption}[Bounded inputs]
\label{ass:inputs}
$\norm{x_t}_2 \le R_x$ for all $t$ and all samples.
\end{assumption}

\begin{assumption}[Bounded blocks and thresholds]
\label{ass:weights}
There exist constants $M_A,M_b,M_\theta,M_{\text{out}},M_{b,\text{out}}$ such that for all layers,
\begin{equation}
\begin{aligned}
\norm{A^{(\ell)}}_2 &\le M_A, &
\norm{b^{(\ell)}}_2 &\le M_b, &
\norm{\theta^{(\ell)}}_\infty &\le M_\theta, \\
\norm{W_{\text{out}}}_2 &\le M_{\text{out}}, &
\norm{b_{\text{out}}}_2 &\le M_{b,\text{out}}.
\end{aligned}
\end{equation}
If the decomposition $A^{(\ell)}=P^{(\ell)}B^{(\ell)}W^{(\ell)}$ is used, assume additionally that $\norm{B^{(\ell)}}_2$ is uniformly bounded and $\norm{P^{(\ell)}}_2\le 1$. In practice these conditions can be encouraged through weight decay, spectral normalization, or explicit operator-norm control \cite{miyato2018spectral,gouk2021lipschitz}.
\end{assumption}

\begin{assumption}[Direct parameterization]
\label{ass:param}
Each trainable block is directly parameterized by its entries. Concretely, $A^{(\ell)}$, $b^{(\ell)}$, $\theta^{(\ell)}$, $W_{\text{out}}$, and $b_{\text{out}}$ are identified with their vectorizations, and the Euclidean parameter norm is the corresponding Frobenius / $\ell_2$ norm.
\end{assumption}

\begin{assumption}[One-step contraction]
\label{ass:gamma}
Let
\begin{equation}
\gamma \defeq \alpha + M_\theta B_1.
\end{equation}
Assume $\gamma<1$.
\end{assumption}

\begin{remark}
Assumption~\ref{ass:gamma} is a sufficient contraction condition for the one-step surrogate state update because $\norm{\diag(\theta^{(\ell)})}_2=\norm{\theta^{(\ell)}}_\infty$. The resulting bounds are conservative, but they make the temporal gain finite and explicit.
\end{remark}

\begin{remark}[Run-level contraction diagnostic]
The theory uses the worst-case constant $M_\theta$. In experiments, we recommend also reporting the observed quantity
\begin{equation}
\hat M_\theta \defeq \max_{\ell,k}\norm{\theta_k^{(\ell)}}_\infty,
\qquad
\hat\gamma \defeq \alpha + \hat M_\theta B_1,
\end{equation}
where $k$ indexes saved checkpoints or training iterations. This does not prove tightness of the theorem, but it does reveal whether the sufficient contraction condition is numerically plausible for the trained theorem-aligned runs.
\end{remark}

For cross-entropy, no additional loss-regularity assumption is needed. The logit gradient and Hessian are uniformly bounded: for all $o\in\R^C$ and labels $y$,
\begin{equation}
\begin{aligned}
\norm{\nabla_o \ell(o,y)}_2 &\le \sqrt{2},\\
\norm{\nabla_o \ell(o,y)}_1 &\le 2,\\
\norm{\nabla_o^2 \ell(o,y)}_2 &\le \frac{1}{2}.
\end{aligned}
\label{eq:ce_bounds}
\end{equation}

\subsection{State stability, smoothness, and robustness}
\begin{proposition}[State stability and input Lipschitz continuity]
\label{prop:state_and_input}
Under Assumptions~\ref{ass:inputs}--\ref{ass:gamma} and Definition~\ref{def:surrogate}, the surrogate states are uniformly bounded. In particular, for every layer $\ell$ and time $t$,
\begin{equation}
\norm{u_t^{(\ell)}}_2 \le R_u^{(\ell)},
\end{equation}
with one valid explicit choice
\begin{equation}
R_u^{(\ell)} \defeq S_T(\alpha)\Big(M_A R_z^{(\ell-1)} + M_b + M_\theta \sqrt{d_\ell}\Big),
\label{eq:Ru_main}
\end{equation}
where $R_z^{(0)}\defeq R_x$ and $R_z^{(\ell-1)}\defeq \sqrt{d_{\ell-1}}$ for $\ell\ge 2$.
Consequently, the surrogate readout is input-Lipschitz:
\begin{equation}
\norm{\tilde f_w(x_{1:T}) - \tilde f_w(x'_{1:T})}_2
\le L_x\,\norm{x_{1:T}-x'_{1:T}}_{2,2},
\end{equation}
with one valid explicit choice
\begin{equation}
L_x \defeq M_{\mathrm{out}}\,\frac{\left(B_1 M_A S_T(\gamma)\right)^L}{\sqrt{T}}.
\label{eq:Lx_main}
\end{equation}
\end{proposition}

\begin{theorem}[Smoothness of the empirical surrogate objective]
\label{thm:smooth_main}
Under Assumptions~\ref{ass:inputs}--\ref{ass:param}, Assumption~\ref{ass:gamma}, and Definition~\ref{def:admissible}, there exist explicit constants $L_w$ and $H_w$ such that
\begin{equation}
\begin{aligned}
\sup_{x_{1:T}} \norm{J_w \tilde f_w(x_{1:T})}_2 &\le L_w,\\
\sup_{x_{1:T}} \max_{1\le c\le C}\norm{\nabla_w^2 \tilde f_{w,c}(x_{1:T})}_2 &\le H_w.
\end{aligned}
\label{eq:LwHw_main}
\end{equation}
Hence the empirical surrogate objective $\tilde L_S$ is $\beta$-smooth with
\begin{equation}
\begin{aligned}
\norm{\nabla \tilde L_S(w)-\nabla \tilde L_S(w')}_2 &\le \beta \norm{w-w'}_2,\\
\beta &\defeq \frac{1}{2}L_w^2 + 2H_w.
\end{aligned}
\label{eq:beta_main}
\end{equation}
Appendix~\ref{app:explicit_constants} gives conservative closed forms for $L_w$ and $H_w$ and makes their dependence on depth, temporal gain, and surrogate slope explicit.
\end{theorem}

\begin{remark}[How to interpret the constants]
The explicit formulas for $L_w$, $H_w$, and $\beta$ are intended to expose qualitative dependence on depth, temporal gain, and surrogate slope. They need not be numerically tight. We therefore recommend pairing the theory with empirical local smoothness diagnostics such as secant estimates around trained checkpoints rather than interpreting the bound as a precise predictor of optimization curvature.
\end{remark}

\begin{corollary}[first order SAM view]
\label{cor:sam_first_order}
If $\tilde L_S$ is $\beta$-smooth, then for any $w$,
\begin{equation}
\tilde L_{\mathrm{SAM}}(w)
\le \tilde L_S(w) + \rho \norm{\nabla \tilde L_S(w)}_2 + \frac{\beta\rho^2}{2}.
\label{eq:sam_upper_main}
\end{equation}
Therefore, for small $\rho$, minimizing $\tilde L_{\mathrm{SAM}}(w)$ approximately minimizes
\begin{equation}
\tilde L_S(w) + \rho \norm{\nabla \tilde L_S(w)}_2,
\end{equation}
up to an additive $O(\beta\rho^2)$ term \cite{foret_sam,wen2023samsharpness}.
\end{corollary}

\begin{assumption}[Stochastic gradients and independent second pass]
\label{ass:sg}
Let $g_B(w)=\nabla \tilde L_B(w)$ denote the minibatch gradient. Assume
\begin{equation}
\E[g_B(w)] = \nabla \tilde L_S(w),
\qquad
\E\norm{g_B(w)-\nabla \tilde L_S(w)}_2^2 \le \sigma^2.
\end{equation}
In SAST, assume the minibatch $B_k$ used to form $\epsilon_k$ is independent of the second minibatch $B_k'$ used to compute $g_{B_k'}(w_k+\epsilon_k)$.
\end{assumption}

\begin{theorem}[Nonconvex convergence of stochastic SAST]
\label{thm:sast_conv}
Assume $\tilde L_S$ is $\beta$-smooth and bounded below by $\tilde L^\star$. Let
\begin{equation}
w_{k+1} = w_k - \eta g_{B_k'}(w_k+\epsilon_k),
\qquad
\epsilon_k = \rho \frac{g_{B_k}(w_k)}{\norm{g_{B_k}(w_k)}_2+\delta}.
\end{equation}
If $\eta\le 1/(4\beta)$, then
\begin{equation}
\frac{1}{K}\sum_{k=0}^{K-1}\E\norm{\nabla \tilde L_S(w_k)}_2^2
\le
\frac{4(\tilde L_S(w_0)-\tilde L^\star)}{\eta K}
+ 3\beta^2\rho^2
+ 2\eta\beta\sigma^2.
\label{eq:conv_rate}
\end{equation}
The proof follows a standard smooth nonconvex stochastic-optimization template, with the SAM perturbation contributing the additive $O(\beta^2\rho^2)$ term \cite{bottou2018optimization}.
\end{theorem}

\begin{proposition}[Loss stability under sequence perturbation]
\label{prop:input_rob_main}
Under Proposition~\ref{prop:state_and_input}, for any $x_{1:T},x'_{1:T}$ and any label $y$,
\begin{equation}
\begin{aligned}
\left|\ell(\tilde f_w(x_{1:T}),y)-\ell(\tilde f_w(x'_{1:T}),y)\right|
&\le \sqrt{2}\,L_x\,\norm{x_{1:T}-x'_{1:T}}_{2,2}.
\end{aligned}
\label{eq:loss_stab_main}
\end{equation}
If $\tilde x_t=m_t\odot x_t$ with independent coordinate drops $\Pr[(m_t)_j=0]=p$, then under Assumption~\ref{ass:inputs},
\begin{equation}
\E\!\left[\norm{x_{1:T}-\tilde x_{1:T}}_{2,2}\right] \le \sqrt{pT}\,R_x,
\end{equation}
and therefore
\begin{equation}
\E\!\left|\ell(\tilde f_w(x_{1:T}),y)-\ell(\tilde f_w(\tilde x_{1:T}),y)\right|
\le \sqrt{2}\,L_x\sqrt{pT}\,R_x.
\end{equation}
\end{proposition}

\noindent The main-text theory stops here. Appendix~\ref{app:mechanism} adds a \emph{local} per-sample mechanism proposition: if the parameter-to-logit Jacobian is well conditioned at a given example, then smaller per-sample parameter gradients upper bound smaller input gradient norms. We use that result only to motivate diagnostics; it is not needed for the unconditional guarantees above.

\section{Experiments}
\label{sec:experiments}
\noindent We organize the experiments around five practical questions: (i) Does SAST preserve clean surrogate forward accuracy? (ii) Does it improve hard spike transfer? (iii) Does it improve corruption robustness? (iv) Are the gains larger than what can be explained by extra compute alone? (v) Are the theorem assumptions at least numerically plausible on the runs for which we invoke them?

\subsection{Benchmarks, protocol, and theory alignment}
\noindent We report \textbf{N-MNIST} \cite{nmnist} and \textbf{DVS Gesture} \cite{amir2017lowpower}. N-MNIST makes the surrogate-to-hard transfer gap easy to resolve, while DVS Gesture provides a stronger benchmark. When additional compute is available, CIFAR10-DVS can be included as an appendix extension. Each event stream is converted to a sequence of $T$ frame-like tensors using temporal binning or a standard event-to-frame transform \cite{tonic}. Inputs are normalized to $[0,1]$ per pixel and channel so that Assumption~\ref{ass:inputs} holds with $R_x\le\sqrt{d_0}$.

\noindent Unless otherwise noted, training uses the surrogate forward model in Definition~\ref{def:surrogate}. theory aligned runs use average pooling or strided convolutions and avoid training-mode BatchNorm in the main architecture. The hard spike numbers in the main tables always follow Definition~\ref{def:hard_eval}: the same trained parameters are used, states are reset per sequence, and only the spike nonlinearity is replaced.

\begin{table}[t]
\caption{Which experimental settings are covered by the theory}
\label{tab:theory_alignment}
\centering
\footnotesize
\setlength{\tabcolsep}{2.8pt}
\resizebox{\columnwidth}{!}{%
\begin{tabular}{lccc}
\toprule
Setting & surrogate forward train & Smooth blocks & Covered by theory \\
\midrule
Main N-MNIST / DVS runs & Yes & AvgPool / affine norm & Yes \\
MaxPool / BatchNorm extension & Yes & No & No (practical only) \\
hard forward / surrogate backward baseline & No & Mixed & No \\
hard spike inference only & Training unchanged & -- & No new training guarantee \\
\bottomrule
\end{tabular}%
}
\end{table}

\paragraph{Baselines and fairness.}
Each reported SAST result should be paired with three controls: (1) baseline surrogate forward training without SAM, (2) an ablation with reused second minibatch $B'=B$, and (3) a \emph{compute-matched} baseline trained with the same approximate number of forward/backward passes as SAST. The third control is essential because SAM doubles the gradient-evaluation budget. Additional controls used to interpret the transfer gap are specified in Section~\ref{subsec:required_controls}.

\subsection{Main results}
\noindent Table~\ref{tab:main_results} reports the primary accuracy metrics together with the transfer gap
\begin{equation}
\Delta_{\mathrm{transfer}} \defeq \mathrm{Acc}_{\mathrm{sur}} - \mathrm{Acc}_{\mathrm{hard}}.
\end{equation}
For this paper, the primary empirical goal is to minimize $\Delta_{\mathrm{transfer}}$, while preserving strong surrogate forward clean accuracy.

\begin{table}[t]
\caption{Main accuracy results}
\label{tab:main_results}
\centering
\footnotesize
\setlength{\tabcolsep}{3.0pt}
\resizebox{\columnwidth}{!}{%
\begin{tabular}{llccc}
\toprule
Dataset & Method & Surrogate forward & Hard spike & Transfer gap \\
\midrule
N-MNIST & Baseline surrogate training & 0.9606 $\pm$ 0.0033 & 0.6572 $\pm$ 0.0974 & 0.3034 \\
N-MNIST & SAST ($\rho = 0.02$) & 0.9789 $\pm$ 0.0033 & 0.7453 $\pm$ 0.0251 & 0.2336 \\
N-MNIST & SAST ($\rho = 0.05$) & \textbf{0.9831 $\pm$ 0.0007} & 0.7812 $\pm$ 0.0847 & 0.2019 \\
N-MNIST & SAST ($\rho = 0.10$) & 0.9786 $\pm$ 0.0009 & 0.8335 $\pm$ 0.0894 & 0.1451 \\
N-MNIST & SAST ($\rho = 0.20$) & 0.9753 $\pm$ 0.0013 & 0.9018 $\pm$ 0.0496 & 0.0735 \\
N-MNIST & SAST ($\rho = 0.30$) & 0.9721 $\pm$ 0.0012 & \textbf{0.9473 $\pm$ 0.0462} & \textbf{0.0248} \\
N-MNIST & SAST ($\rho = 0.40$) & 0.9697 $\pm$ 0.0009 & 0.9424 $\pm$ 0.0332 & 0.0273 \\
N-MNIST & SAST ($\rho = 0.50$) & 0.9655 $\pm$ 0.0013 & 0.9318 $\pm$ 0.0323 & 0.0337 \\
\midrule
DVS Gesture & Baseline surrogate training & 0.7502 $\pm$ 0.0142 & 0.3182 $\pm$ 0.0732 & 0.4320 \\
DVS Gesture & SAST ($\rho = 0.02$) & 0.7670 $\pm$ 0.0134 & 0.4489 $\pm$ 0.0716 & 0.3181 \\
DVS Gesture & SAST ($\rho = 0.05$) & 0.7519 $\pm$ 0.0187 & 0.4754 $\pm$ 0.0623 & 0.2765 \\
DVS Gesture & SAST ($\rho = 0.10$) & 0.7727 $\pm$ 0.0107 & 0.5341 $\pm$ 0.0357 & 0.2386 \\
DVS Gesture & SAST ($\rho = 0.20$) & \textbf{0.8087 $\pm$ 0.0043} & 0.5957 $\pm$ 0.0116 & 0.2130 \\
DVS Gesture & SAST ($\rho = 0.30$) & 0.7778 $\pm$ 0.0076 & 0.5926 $\pm$ 0.0151 & 0.1852 \\
DVS Gesture & SAST ($\rho = 0.40$) & 0.7685 $\pm$ 0.0151 & \textbf{0.6327 $\pm$ 0.0116} & \textbf{0.1358} \\
DVS Gesture & SAST ($\rho = 0.50$) & 0.7747 $\pm$ 0.0430 & 0.6022 $\pm$ 0.0191 & 0.1725 \\
\bottomrule
\end{tabular}%
}
\end{table}

The general trend seen in results \ref{tab:main_results} is that minimal returns are seen on the surrogate forward with respect to the extra computation per epoch with the main key notible results being the strongest effect is not better surrogate forward accuracy but a much smaller transfer gap: $\Delta_{\mathrm{transfer}}$ falls from $0.3034$ to $0.0248$, a relative reduction of roughly $91.8\%$, while hard spike accuracy also increases for decreases in $\Delta_{\mathrm{transfer}}$. On DVS Gesture, the same trend remains but is more modest: the transfer gap drops from $0.4320$ to $0.1358$ and hard spike accuracy improves by $31.45$ points. We therefore interpret SAST primarily as a transfer gap reduction method rather than as a guarantee of SOTA hard spike accuracy.

A second observation is seed sensitivity at small radius. On N-MNIST, the hard spike standard deviation at $\rho=0.02$ is substantially larger than subsequently values of $\rho > 0.02$, suggesting that mild perturbation radii may not reliably move all runs into the same flatter regime. To make that behavior transparent, the appendix should report per-seed transfer gaps and per-seed hard spike accuracies, not only means and standard deviations.

\subsection{Robustness to activity drop and temporal corruption}
\noindent Figure~\ref{fig:robustness_placeholder} reports accuracy under activity-drop and temporal corruptions with severity $p\in\{0,0.1,0.2,0.3,0.4\}$ for both surrogate forward and hard spike evaluation. To connect with Proposition~\ref{prop:input_rob_main}, one corruption family uses independent event drop. A second family stresses temporal structure (e.g., time jitter or dropped time bins) using the same severity grid across methods. In addition to pointwise curves, we recommend summarizing each curve by the mean accuracy across severities and by the area under the accuracy--severity curve.

\begin{figure}[t]
\centering
\IfFileExists{chartpng.png}{%
\includegraphics[width=\linewidth]{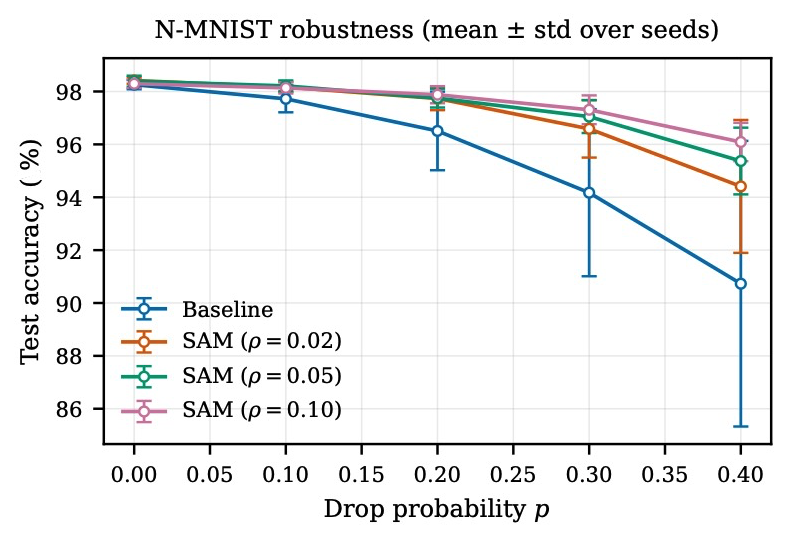}%
}{%
\placeholderbox{Insert robustness plot: accuracy versus corruption severity $p$ for both surrogate forward and hard spike evaluation.}%
}
\caption{Robustness on N-MNIST under random event-drop corruption. Test accuracy (mean $\pm$ std over seeds) is plotted against drop probability $p$ for both surrogate forward and hard spike evaluation; SAST trained models are expected to degrade more gracefully than the baseline as corruption increases, with moderate-to-large radii typically showing the strongest gains at high drop rates.}
\label{fig:robustness_placeholder}
\end{figure}

\subsection{Sharpness, transfer, sensitivity, and theory-to-practice diagnostics}
To support the SAM mechanism discussion, we report at least one sharpness proxy and several held-out diagnostics: (i) the SAM gap $\Delta_\rho=\tilde L(w+\epsilon^*)-\tilde L(w)$ with $\norm{\epsilon^*}_2=\rho$, (ii) the transfer gap $\Delta_{\mathrm{transfer}}$, (iii) the mean per-example parameter-gradient norm $\norm{\nabla_w \ell_i}_2$, and (iv) the mean per-example input gradient norm $\norm{\nabla_{x_{1:T}}\ell_i}_2$. The gradient diagnostics are computed per example on a held-out set because the local mechanism proposition in Appendix~\ref{app:mechanism} is per-sample.

To tie the theory to practice more directly, we also recommend two descriptive diagnostics on theorem-aligned runs: (v) the observed contraction factor $\hat\gamma = \alpha + \hat M_\theta B_1$, and (vi) a local secant smoothness estimate,
\begin{equation}
\hat \beta_{\mathrm{sec}}(w) \defeq \max_{1\le j\le m}
\frac{\norm{\nabla \tilde L_S(w+\delta_j)-\nabla \tilde L_S(w)}_2}{\norm{\delta_j}_2},
\label{eq:beta_sec}
\end{equation}
for random perturbations $\delta_j$ of small norm around selected checkpoints. These quantities do not validate the theorems, but they do reveal whether the sufficient conditions and explicit constants are wildly disconnected from the trained models.

For the appendix mechanism proposition, we additionally report the smallest singular value of the logit-space Gram matrix $J_w J_w^\top$ on held-out examples. Since $J_w\in\R^{C\times p}$ with $p\gg C$ in most modern models, this tests \emph{row-space conditioning} rather than square invertibility.

\subsection{Additional controls required for final practical assessment}
\label{subsec:required_controls}
The current results establish that SAST can substantially reduce the surrogate to hard transfer gap on the two reported datasets. Stronger claims about practical superiority, however, require the controls in Table~\ref{tab:controls}. We list them explicitly so that they are part of the paper's evaluation protocol rather than implicit future work.

\begin{table*}[t]
\caption{Additional controls and extra experiments required for a final assessment of SAST.}
\label{tab:controls}
\centering
\footnotesize
\begin{tabularx}{\textwidth}{>{\raggedright\arraybackslash}p{0.18\textwidth} >{\raggedright\arraybackslash}X >{\raggedright\arraybackslash}X}
\toprule
Control experiment & Protocol & Why it matters \\
\midrule
Compute matched baseline & Match the total number of forward and backward passes used by SAST, either by doubling baseline updates or by extending baseline training until pass counts are comparable. Report surrogate forward accuracy, hard spike accuracy, transfer gap, robustness summary metrics, and wall clock. & Separates the effect of sharpness aware updates from the effect of extra optimization compute. \\

Reused second minibatch & Replace the independent second minibatch by $B'=B$ while keeping everything else fixed. & Tests the practical cost of violating the independence assumption used in the convergence theorem. \\

ASAM baseline & Replace SAM by ASAM with matched optimizer settings and matched compute. & Checks whether scale adaptive perturbations outperform or match vanilla SAM in this setting. \\

Post training threshold calibration & With network weights fixed, tune either a single global threshold multiplier $\lambda$ or per-layer multipliers $\lambda_\ell$ on a validation set to maximize hard spike accuracy. Report both the best calibration-only result and the combined result when calibration is applied to SAST trained models. & Tests whether a cheaper calibration heuristic can close the transfer gap without two-pass training. \\

Surrogate choice and slope & Compare arctan and fast-sigmoid surrogates over a small slope grid $k$. For each run, report $k$, $B_1$, $\hat M_\theta$, and $\hat\gamma$. & Determines whether the gains are specific to one surrogate family or one operating regime, and shows whether theorem aligned runs actually satisfy the contraction diagnostic. \\

Theory misaligned extension & Repeat the strongest setting with MaxPool and BatchNorm. Clearly label these as practical extensions not covered by the smoothness theory. & Separates theorem aligned evidence from broader engineering usefulness. \\

Seed sensitivity analysis & Plot per-seed transfer gaps, hard spike accuracies, and threshold statistics. Report medians in addition to means when variance is large. & Clarifies whether the gain comes from a consistent shift in all runs or from eliminating a subset of pathological seeds. \\
\bottomrule
\end{tabularx}
\end{table*}

\subsection{Training cost}
Because SAM requires two gradient evaluations, we report per-epoch wall-clock time, peak memory, and relative cost factors for baseline and SAST. When SAST overhead is partly offset by earlier stopping or fewer total epochs, that effect should be reported quantitatively rather than informally.

\begin{table}[t]
\caption{Training overhead measured on the hardware used for the experiments.}
\label{tab:overhead}
\centering
\footnotesize
\setlength{\tabcolsep}{2.8pt}
\resizebox{\columnwidth}{!}{%
\begin{tabular}{llcccc}
\toprule
Dataset & Method & Time / epoch & Time factor & Peak memory & Mem. factor \\
\midrule
N-MNIST & Baseline surrogate & $80.48 \pm 0.25$ s & $1.00\times$ & $0.40 \pm 0.01$ GB & $1.00\times$ \\
N-MNIST & SAST ($\rho=k, k>0$) & $167.72 \pm 1.81$ s & $2.08\times$ & $0.40 \pm 0.01$ GB & $1.00\times$ \\
DVS Gesture & Baseline surrogate & $2.87 \pm 0.07$ s & $1.00\times$ & $2.48 \pm 0.01$ GB & $1.00\times$ \\
DVS Gesture & SAST ($\rho=k, k>0$) & $5.09 \pm 0.09$ s & $1.77\times$ & $2.48 \pm 0.01$ GB & $1.00\times$ \\
\bottomrule
\end{tabular}%
}
\end{table}

\section{Discussion}
The main message is intentionally narrow. Unconditionally, SAST optimizes a sharpness aware surrogate objective in parameter space and inherits standard smooth nonconvex guarantees under the surrogate forward model. Robustness to bounded input corruption follows from an explicit input Lipschitz constant, although the constant is conservative. The stronger claim that flatter parameter-space solutions also reduce input sensitivity is justified only locally through the per sample Jacobian conditioning proposition in Appendix~\ref{app:mechanism}; we therefore present it as a mechanism result rather than a universal theorem.

The empirical evidence is strongest on surrogate-to-hard transfer. On both datasets, SAST narrows the transfer gap substantially while leaving surrogate forward accuracy similar. That is a meaningful result for deployment scenarios in which the final model must use hard spikes. At the same time, the current evidence does \emph{not} yet show that SAST is the most efficient way to obtain that benefit: compute-matched baselines, calibration baselines, and ASAM comparisons remain necessary for a full practical verdict.

The mixed behavior of input gradient norms should also be interpreted carefully. Those diagnostics matter for the local mechanism discussion, but the central empirical claim of this paper is narrower: SAST reduces sharpness proxies and improves hard spike transfer in the reported regimes. If some radii fail to reduce input gradients consistently, that weakens the mechanism narrative, not the core transfer-gap observation.

\section{Limitations}
The present theory is intentionally restrictive. The smoothness analysis assumes surrogate forward training, average pooling or other smooth linear downsampling, and an inference-mode affine normalization model. If the implementation uses hard spikes in the forward pass, MaxPool, or training mode BatchNorm, the analysis no longer applies directly. The explicit constants are also conservative and may be numerically loose for realistic network sizes and sequence lengths.

The empirical scope remains limited in three important ways. First, the current benchmarks are only N-MNIST and DVS Gesture; broader claims would benefit from CIFAR10-DVS or another harder event-based benchmark. Second, the present results do not yet include the compute-matched, threshold-calibration, and ASAM controls needed to decide whether SAST is preferable to cheaper alternatives. Third, the DVS Gesture hard spike accuracy remains modest in absolute terms, so the method should be viewed as a promising transfer gap reduction technique rather than a complete solution to hard spike deployment.

\section{Conclusion}
We presented SAST, a sharpness aware optimization method for surrogate forward training of SNNs. By training an auxiliary smooth surrogate forward network, the method turns SNN optimization into a smooth empirical risk problem and supports explicit state stability, input-Lipschitz, smoothness, first order SAM, and stochastic convergence guarantees. Empirically, the clearest current effect is improved surrogate to hard transfer: on the reported benchmarks, SAST preserves surrogate forward accuracy while substantially improving hard spike evaluation. The next step is to determine, through compute-matched and calibration-aware comparisons, whether those gains remain compelling relative to cheaper baselines.

\appendices
\section{Additional Theory Details and Proof Sketches}
\label{app:theory}

\subsection{Explicit constants for the smoothness theorem}
\label{app:explicit_constants}
Let $a_t^{(\ell)} \defeq u_t^{(\ell)}-\theta^{(\ell)}$, $D_t^{(\ell)} \defeq \diag(\sigma'(a_t^{(\ell)}))$, and $E_t^{(\ell)} \defeq \diag(\sigma''(a_t^{(\ell)}))$. By Definition~\ref{def:admissible},
\begin{equation}
\norm{D_t^{(\ell)}}_2 \le B_1,
\qquad
\norm{E_t^{(\ell)}}_2 \le B_2.
\end{equation}

\begin{lemma}[One-step Jacobian bounds]
\label{lem:local_jac}
Under the surrogate forward dynamics,
\begin{align}
\norm{\frac{\partial u_t^{(\ell)}}{\partial u_{t-1}^{(\ell)}}}_2 &\le \alpha + M_\theta B_1, \label{eq:juu}\\
\norm{\frac{\partial \tilde z_t^{(\ell)}}{\partial u_t^{(\ell)}}}_2 &\le B_1, \label{eq:jzu}\\
\norm{\frac{\partial u_t^{(\ell)}}{\partial z_t^{(\ell-1)}}}_2 &\le M_A. \label{eq:juz}
\end{align}
\end{lemma}

\begin{lemma}[Conservative parameter-Jacobian and Hessian constants]
\label{lem:Jw_Hw_appendix}
Under Assumptions~\ref{ass:inputs}--\ref{ass:param}, Assumption~\ref{ass:gamma}, and Definition~\ref{def:admissible}, one conservative choice in \eqref{eq:LwHw_main} is
\begin{equation}
L_w \defeq M_{\text{out}} \cdot \frac{\left(B_1 M_A S_T(\gamma)\right)^L}{\sqrt{T}} \cdot C_p,
\label{eq:Lw_appendix}
\end{equation}
where
\begin{equation}
C_p \defeq \left(\sum_{\ell=1}^L \left(R_z^{(\ell-1)} + 2 + M_\theta B_1\right) + \sqrt{d_L} + 1\right),
\label{eq:Cp_appendix}
\end{equation}
and
\begin{equation}
C_{\text{inner}} \defeq \sum_{\ell=1}^L \left(R_z^{(\ell-1)} + 2 + M_\theta B_1\right).
\label{eq:Cinner_appendix}
\end{equation}
A conservative componentwise Hessian constant is obtained by defining
$G_T \defeq B_1 M_A S_T(\gamma)$ and writing
\begin{equation}
\begin{aligned}
H_w
&\defeq M_{\text{out}}\,G_T^{2L}\,B_2\,C_p^2
\\
&\quad + G_T^{L}\,\frac{C_{\text{inner}}}{\sqrt{T}}.
\end{aligned}
\label{eq:Hw_appendix}
\end{equation}
That is,
\begin{equation}
\sup_{x_{1:T}} \max_{1\le c\le C}\norm{\nabla_w^2 \tilde f_{w,c}(x_{1:T})}_2 \le H_w.
\end{equation}
The constants are intentionally conservative; their main purpose is to expose how depth $L$, temporal gain $S_T(\gamma)$, and surrogate slope $(B_1,B_2)$ enter the bound.
\end{lemma}

\subsection{Local mechanism proposition}
\label{app:mechanism}
The previous results bound robustness through explicit Lipschitz constants. To connect SAM more directly to input sensitivity, we introduce a local conditioning assumption on the parameter-to-logit Jacobian.

\begin{assumption}[Local Jacobian conditioning]
\label{ass:cond}
Fix a data point $(x_{1:T},y)$. Let $J_w(x_{1:T})\defeq J_w \tilde f_w(x_{1:T})\in\R^{C\times p}$. Assume there exists $\mu>0$ such that
\begin{equation}
\sigma_{\min}\big(J_w(x_{1:T})\big) \ge \mu.
\label{eq:mu}
\end{equation}
\end{assumption}

\begin{remark}[Interpretation of Assumption~\ref{ass:cond}]
Since $J_w\in\R^{C\times p}$ with $p$ typically much larger than $C$, Assumption~\ref{ass:cond} requires \emph{full row rank} and a nondegenerate local lower singular-value bound on the logit row space; it does not require square invertibility. We do not assume such a bound globally. It is used only as a local explanatory condition and should be paired with empirical conditioning diagnostics if invoked in interpretation.
\end{remark}

\begin{proposition}[Local link between per-sample parameter and input gradients]
\label{prop:param_to_input_grad}
Let $v(x_{1:T},y)\defeq \nabla_o \ell(o,y)\vert_{o=\tilde f_w(x_{1:T})}\in\R^C$. Then
\begin{equation}
\norm{\nabla_{x_{1:T}} \ell(\tilde f_w(x_{1:T}),y)}_2 \le \norm{J_{x_{1:T}}(x_{1:T})}_2\,\norm{v(x_{1:T},y)}_2,
\end{equation}
and
\begin{equation}
\norm{\nabla_w \ell(\tilde f_w(x_{1:T}),y)}_2 \ge \sigma_{\min}(J_w(x_{1:T}))\,\norm{v(x_{1:T},y)}_2.
\end{equation}
If Assumption~\ref{ass:cond} holds, let
$\ell_w \defeq \ell(\tilde f_w(x_{1:T}),y)$. Then
\begin{equation}
\norm{\nabla_{x_{1:T}} \ell_w}_2
\le
\frac{\norm{J_{x_{1:T}}(x_{1:T})}_2}{\mu}\,\norm{\nabla_w \ell_w}_2.
\label{eq:grad_link}
\end{equation}
Consequently, any training procedure that reduces \emph{per-sample} parameter-gradient norms while leaving the local conditioning ratio comparable yields a matching upper bound reduction on per-sample input gradient norms at that point. This is a local explanatory statement, not a global theorem about SAM.
\end{proposition}

\subsection{Proof sketches}
\begin{proof}[Proof sketch for Proposition~\ref{prop:state_and_input}]
For $\ell=1$, $z_t^{(0)}=x_t$ and Assumption~\ref{ass:inputs} gives $\norm{z_t^{(0)}}_2\le R_x$. For $\ell\ge 2$, $\tilde z_t^{(\ell-1)}\in[0,1]^{d_{\ell-1}}$, hence $\norm{\tilde z_t^{(\ell-1)}}_2\le \sqrt{d_{\ell-1}}$. Likewise, $\tilde s_{t-1}^{(\ell)}\in[0,1]^{d_\ell}$ implies $\norm{\tilde s_{t-1}^{(\ell)}}_2\le \sqrt{d_\ell}$. From the membrane recursion,
\begin{equation*}
\begin{aligned}
\norm{u_t^{(\ell)}}_2
&\le
\alpha \norm{u_{t-1}^{(\ell)}}_2 + M_A \norm{\tilde z_t^{(\ell-1)}}_2 \\
&\quad + \norm{b^{(\ell)}}_2 + \norm{\diag(\theta^{(\ell)})}_2\,\norm{\tilde s_{t-1}^{(\ell)}}_2.
\end{aligned}
\end{equation*}
Since $\norm{\diag(\theta^{(\ell)})}_2=\norm{\theta^{(\ell)}}_\infty\le M_\theta$, unrolling the scalar recursion yields \eqref{eq:Ru_main}.

For the input-Lipschitz bound, define $\|v_{1:T}\|_{2,2}\defeq (\sum_{t=1}^T \|v_t\|_2^2)^{1/2}$. Let $\Delta u_t^{(\ell)}$ and $\Delta z_t^{(\ell-1)}$ denote the differences induced by two input sequences. Lemma~\ref{lem:local_jac} gives the one-step recursion
\begin{equation*}
\|\Delta u_t^{(\ell)}\|_2 \le \gamma\,\|\Delta u_{t-1}^{(\ell)}\|_2 + M_A\,\|\Delta z_t^{(\ell-1)}\|_2,
\qquad \gamma=\alpha+M_\theta B_1.
\end{equation*}
Unrolling over time and applying Young's convolution inequality yields
\begin{equation*}
\|\Delta u_{1:T}^{(\ell)}\|_{2,2}
\le M_A S_T(\gamma)\,\|\Delta z_{1:T}^{(\ell-1)}\|_{2,2}.
\end{equation*}
Since $\sigma$ is $B_1$-Lipschitz,
\begin{equation*}
\|\Delta \tilde z_{1:T}^{(\ell)}\|_{2,2}
\le B_1 M_A S_T(\gamma)\,\|\Delta z_{1:T}^{(\ell-1)}\|_{2,2}.
\end{equation*}
Applying this bound across $L$ layers gives
\begin{equation*}
\|\Delta \tilde z_{1:T}^{(L)}\|_{2,2}
\le \left(B_1 M_A S_T(\gamma)\right)^L \|\Delta x_{1:T}\|_{2,2}.
\end{equation*}
Finally, the time-average readout contributes a factor $1/\sqrt{T}$ by Cauchy--Schwarz, and the linear readout contributes $M_{\text{out}}$, yielding \eqref{eq:Lx_main}.
\end{proof}

\begin{proof}[Proof sketch for Theorem~\ref{thm:smooth_main}]
For one sample,
\begin{equation*}
\nabla_w \ell(\tilde f_w(x_{1:T}),y) = (J_w\tilde f_w(x_{1:T}))^\top \nabla_o \ell(o,y)\vert_{o=\tilde f_w(x_{1:T})}.
\end{equation*}
Differentiating again gives
\begin{equation*}
\nabla_w^2 \ell
= (J_w\tilde f)^\top \nabla_o^2\ell\,(J_w\tilde f)
+ \sum_{c=1}^C \frac{\partial \ell}{\partial o_c}\,\nabla_w^2 \tilde f_c.
\end{equation*}
Using \eqref{eq:ce_bounds}, \eqref{eq:LwHw_main}, and $\|\nabla_o \ell\|_1\le 2$ gives
\begin{equation*}
\|\nabla_w^2 \ell\|_2 \le \frac{1}{2}L_w^2 + 2H_w.
\end{equation*}
Averaging over samples preserves the same upper bound, so $\tilde L_S$ is $\beta$-smooth with $\beta$ given by \eqref{eq:beta_main}. The explicit conservative formulas follow from Lemma~\ref{lem:Jw_Hw_appendix}.
\end{proof}

\begin{proof}[Proof sketch for Corollary~\ref{cor:sam_first_order}]
By smoothness,
\begin{equation*}
\tilde L_S(w+\epsilon)
\le
\tilde L_S(w) + \ip{\nabla \tilde L_S(w)}{\epsilon} + \frac{\beta}{2}\norm{\epsilon}_2^2.
\end{equation*}
Maximizing the right-hand side over $\norm{\epsilon}_2\le \rho$ yields \eqref{eq:sam_upper_main}.
\end{proof}

\begin{proof}[Proof sketch for Theorem~\ref{thm:sast_conv}]
Let $L(w)\defeq \tilde L_S(w)$ and define
\begin{equation*}
a_k \defeq \nabla L(w_k),
\qquad
b_k \defeq \nabla L(w_k+\epsilon_k).
\end{equation*}
Because $B_k'$ is independent of $B_k$, the second-pass gradient is conditionally unbiased:
\begin{equation*}
\E[g_{B_k'}(w_k+\epsilon_k)\mid w_k,\epsilon_k] = b_k.
\end{equation*}
By $\beta$-smoothness of $L$,
\begin{equation*}
\begin{aligned}
L(w_{k+1})
&\le L(w_k) - \eta\ip{a_k}{g_{B_k'}(w_k+\epsilon_k)}
\\
&\quad + \frac{\beta\eta^2}{2}\norm{g_{B_k'}(w_k+\epsilon_k)}_2^2.
\end{aligned}
\end{equation*}
Taking conditional expectation and using the variance bound gives
\begin{equation}
\E[L(w_{k+1})\mid w_k,\epsilon_k]
\le
L(w_k) - \eta\ip{a_k}{b_k} + \frac{\beta\eta^2}{2}\left(\norm{b_k}_2^2+\sigma^2\right).
\label{eq:conv_pf_1}
\end{equation}
Now $\norm{b_k-a_k}_2\le \beta\norm{\epsilon_k}_2\le \beta\rho$, so
\begin{equation}
\ip{a_k}{b_k}
= \norm{a_k}_2^2 + \ip{a_k}{b_k-a_k}
\ge \frac{1}{2}\norm{a_k}_2^2 - \frac{1}{2}\beta^2\rho^2,
\label{eq:conv_pf_2}
\end{equation}
and
\begin{equation}
\norm{b_k}_2^2
\le 2\norm{a_k}_2^2 + 2\beta^2\rho^2.
\label{eq:conv_pf_3}
\end{equation}
Substituting \eqref{eq:conv_pf_2} and \eqref{eq:conv_pf_3} into \eqref{eq:conv_pf_1} yields
\begin{equation*}
\begin{aligned}
\E[L(w_{k+1})\mid w_k,\epsilon_k]
&\le
L(w_k) - \left(\frac{\eta}{2}-\beta\eta^2\right)\norm{a_k}_2^2 \\
&\quad + \left(\frac{\eta}{2}+\beta\eta^2\right)\beta^2\rho^2
 + \frac{\beta\eta^2}{2}\sigma^2.
\end{aligned}
\end{equation*}
If $\eta\le 1/(4\beta)$, then $\eta/2-\beta\eta^2\ge \eta/4$ and $\eta/2+\beta\eta^2\le 3\eta/4$. Taking full expectation and summing for $k=0,\dots,K-1$ gives
\begin{equation*}
\frac{\eta}{4}\sum_{k=0}^{K-1}\E\norm{\nabla L(w_k)}_2^2
\le
L(w_0)-L^\star + \frac{3\eta K}{4}\beta^2\rho^2 + \frac{\beta\eta^2 K}{2}\sigma^2.
\end{equation*}
Divide by $\eta K/4$ to obtain \eqref{eq:conv_rate}.
\end{proof}

\begin{proof}[Proof sketch for Proposition~\ref{prop:input_rob_main}]
Cross-entropy is $\sqrt{2}$-Lipschitz in logit space by \eqref{eq:ce_bounds}. Combining that fact with Proposition~\ref{prop:state_and_input} yields \eqref{eq:loss_stab_main}. For independent event drop, Jensen's inequality gives
\begin{equation*}
\E\left[\norm{x_{1:T}-\tilde x_{1:T}}_{2,2}\right]
\le
\left(\E\sum_{t=1}^T \norm{x_t-\tilde x_t}_2^2\right)^{1/2}
\le \sqrt{pT}\,R_x,
\end{equation*}
which proves the final claim.
\end{proof}

\begin{proof}[Proof sketch for Proposition~\ref{prop:param_to_input_grad}]
By the chain rule,
\begin{equation*}
\begin{aligned}
\nabla_{x_{1:T}}\ell
&= J_{x_{1:T}}(x_{1:T})^\top v(x_{1:T},y),\\
\nabla_w\ell
&= J_w(x_{1:T})^\top v(x_{1:T},y).
\end{aligned}
\end{equation*}
Therefore
\begin{equation*}
\norm{\nabla_{x_{1:T}}\ell}_2 \le \norm{J_{x_{1:T}}(x_{1:T})}_2\,\norm{v(x_{1:T},y)}_2,
\end{equation*}
while
\begin{equation*}
\begin{aligned}
\norm{\nabla_w\ell}_2
&= \norm{J_w(x_{1:T})^\top v(x_{1:T},y)}_2
\\
&\ge \sigma_{\min}(J_w(x_{1:T}))\,\norm{v(x_{1:T},y)}_2.
\end{aligned}
\end{equation*}
Combining the two inequalities with Assumption~\ref{ass:cond} gives \eqref{eq:grad_link}.
\end{proof}

\section{Extended Experimental Details}
\label{app:expcheck}
This appendix makes the implementation protocol explicit so that Table~\ref{tab:main_results}, Figure~\ref{fig:robustness_placeholder}, and Table~\ref{tab:overhead} can be reproduced with the same evaluation conventions.

\paragraph{Architecture and preprocessing.}
All main results use surrogate-forward LIF dynamics from Definition~\ref{def:surrogate} and the readout in \eqref{eq:readout}. Event streams are converted to fixed-length frame sequences by temporal binning, with polarity preserved as separate channels and inputs normalized to $[0,1]$ as described in Section~\ref{sec:experiments}. For theorem-aligned runs, downsampling is linear (average pooling or strided convolution), and normalization is treated as an affine inference-mode map; MaxPool or training-mode BatchNorm runs are reported only as theory-misaligned practical extensions.

\paragraph{Optimization details.}
Baseline and SAST share the same optimizer, learning-rate schedule, weight decay, augmentation policy, epoch budget, and seed list. The only algorithmic difference is the SAM perturbation: SAST uses the radius grid reported in Table~\ref{tab:main_results}, $\rho\in\{0.02,0.05,0.10,0.20,0.30,0.40,0.50\}$, it should be noted that $\rho$ was tested $\{0.2 \leq \rho \leq 1\}$, with independent minibatches $B$ and $B'$ for the ascent and descent passes. Hidden states are reset at sequence boundaries and before each SAM pass so that perturbation effects are not confounded by stale state carryover.

\paragraph{hard spike evaluation protocol.}
All hard-spike numbers follow Definition~\ref{def:hard_eval} exactly: we keep trained weights and thresholds fixed, reset states per sequence, and replace only the surrogate spike nonlinearity by the hard threshold. Unless explicitly labeled as a calibration baseline, no threshold rescaling, clipping, post-hoc fitting, or test-time statistic adaptation is applied.

\paragraph{Corruption definitions.}
We evaluate robustness on a shared severity grid $p\in\{0,0.1,0.2,0.3,0.4\}$ for every method and evaluation mode. The first corruption family is independent event-drop implemented as Bernoulli masking on input voxels (pixel, time-bin, polarity). The second family is temporal corruption (time jitter or dropped time bins) with boundary-safe handling and polarity preserved. For both families, the same corrupted inputs are used when comparing baseline and SAST to remove sampling noise from method comparisons.

\paragraph{Contraction and smoothness diagnostics.}
For theorem-aligned checkpoints, we report the surrogate slope parameter $k$, the induced bound constant $B_1$, the empirical threshold bound $\hat M_\theta$, and the resulting contraction diagnostic $\hat\gamma=\alpha+\hat M_\theta B_1$. Local smoothness is quantified with the secant estimator in \eqref{eq:beta_sec} on a fixed held-out subset and a shared perturbation-radius grid; reported values include both the empirical secants and the conservative theoretical upper bound so readers can assess looseness directly.

\paragraph{Mechanism diagnostics.}
For Appendix~\ref{app:mechanism}, we compute per-example $\norm{\nabla_w \ell_i}_2$ and $\norm{\nabla_{x_{1:T}}\ell_i}_2$ on a held-out split without augmentation, using the same forward mode as the associated evaluation (surrogate or hard). We additionally estimate the smallest singular value proxy of $J_wJ_w^\top$ to assess local conditioning assumptions. Metrics are reported with mean$\pm$std and median/IQR to distinguish systematic shifts from heavy-tail behavior.

\paragraph{Compute-matched control.}
Because SAST performs two gradient evaluations per update, we define compute matching by total forward/backward pass count. Our primary control doubles baseline updates (or equivalently extends baseline epochs) at fixed batch size until pass counts match the corresponding SAST run; we then report surrogate accuracy, hard-spike accuracy, transfer gap, robustness summaries, and wall-clock cost under the same checkpoint-selection rule.

\paragraph{Threshold-calibration baseline.}
With weights frozen, we evaluate two post-hoc threshold families: a global multiplier $\lambda\in\{0.5,0.6,\dots,1.5\}$ and per-layer multipliers $\lambda_\ell$ tuned by a small validation search. Calibration is selected using validation hard-spike accuracy only, then applied once on test. Calibrated results are always reported as separate baselines and are never merged into the default hard-spike columns.

\paragraph{Seed sensitivity.}
All methods use a shared seed list. In addition to mean $\pm$ std, we report per-seed hard-spike accuracy and transfer-gap points, together with median and interquartile range when variance is high (notably at small $\rho$), so the reader can see whether improvements reflect a broad shift or the removal of a few unstable runs.

\ifanonymous\else
\section*{Acknowledgments}
I would like to thank the numerous people who have given me draft feedback advice.
\fi

\bibliographystyle{IEEEtran}
\bibliography{references}

\end{document}